\documentclass[runningheads]{llncs}
\usepackage{xcolor}  
\usepackage{graphicx}
\usepackage{wrapfig}
\usepackage{amsmath}
\usepackage{breakurl}
\usepackage{adjustbox}
\usepackage[colorlinks = true,
            linkcolor = black,
            urlcolor  = blue,
            citecolor = black,
            anchorcolor = black]{hyperref}
\usepackage{amssymb}
\usepackage{subcaption}
\usepackage{enumitem}
\usepackage[percent]{overpic}
\usepackage{tikz}
\usepackage{pgf-pie}
\usepackage{ifthen}
\usepackage{mathtools}
\usepackage{tabularx}
\usepackage[font={small}]{caption}
\usepackage{pgfplots}
\usepackage{adjustbox}
\newcommand*\circled[1]{\tikz[baseline=(char.base)]{
            \node[shape=circle,draw,inner sep=2pt,fill=white] (char) {#1};}}
\usepackage{xargs}                      
\usepackage{array}
\usepackage{arydshln}
\usepackage{booktabs}
\usepackage{colortbl}
\usepackage{pifont}
\newcommand{\cmark}{\ding{51}}%
\newcommand{\xmark}{\ding{55}}%
\definecolor{darkblue}{RGB}{2,105,164}
\usepackage[colorinlistoftodos,prependcaption,textsize=tiny]{todonotes}
\newcommandx{\unsure}[2][1=]{\todo[linecolor=red,backgroundcolor=red!25,bordercolor=red,#1]{#2}}
\newcommandx{\change}[2][1=]{\todo[linecolor=blue,backgroundcolor=blue!25,bordercolor=blue,#1]{#2}}
\newcommandx{\info}[2][1=]{\todo[linecolor=green,backgroundcolor=green!25,bordercolor=green,#1]{#2}}
\newcommandx{\improvement}[2][1=]{\todo[linecolor=Plum,backgroundcolor=Plum!25,bordercolor=Plum,#1]{#2}}
\newcommandx{\thiswillnotshow}[2][1=]{\todo[disable,#1]{#2}}
\begin{document}
\title{Homological Time Series Analysis of Sensor Signals from Power Plants\thanks{The code can be found at: \href{https://codeberg.org/Jiren/TwirlFlake}{https://codeberg.org/Jiren/TwirlFlake}. \\ This project has been partially supported by Siemens Energy AG. \\ The authors thank Leonie Rumi, Noah Becker, Philipp Gäbelein and the anonymous reviewers for useful suggestions for improvement and proofreading.}}
\titlerunning{Homological Time Series Analysis}
%
\author{Luciano Melodia\orcidID{0000-0002-7584-7287} \and \\
Richard Lenz\orcidID{0000-0003-1551-4824}}
\authorrunning{L. Melodia et al.}
%
\institute{Professorship for Evolutionary Data Management\\
Friedrich-Alexander University Erlangen-Nürnberg\\
91058 Erlangen, Deutschland \\
\email{\{luciano.melodia,richard.lenz\}@fau.de}}
\maketitle              
\begin{abstract}
In this paper, we use topological data analysis techniques to construct a suitable neural network classifier for the task of learning sensor signals of entire power plants according to their reference designation system. We use representations of persistence diagrams to derive necessary preprocessing steps and visualize the large amounts of data. We derive deep architectures with one-dimensional convolutional layers combined with stacked long short-term memories as residual networks suitable for processing the persistence features. We combine three separate sub-networks, obtaining as input the time series itself and a representation of the persistent homology for the zeroth and first dimension. We give a mathematical derivation for most of the used hyper-parameters. For validation, numerical experiments were performed with sensor data from four power plants of the same construction type.

\end{abstract}

\section{Introduction}
Power plants, regardless of their construction, must be intensively maintained and monitored to ensure constantly efficient power generation and to minimize the risk of damage. Sensors measure pressure, temperature, enthalpy, electrical resistance, etc., and their readings are recorded for this reason in order to monitor them. Since power plant operators operate at an international level, there is a need to evaluate information from power plants whose measured values are organized and stored in accordance with the norms that are standard for the specific country. Both the identifiers and the storage structure rarely resemble each other, so signals from power plant sensors must be manually assigned to the appropriate identifiers. These identifiers are called the \emph{power plant reference designation system}, which is defined as an international standard \cite{konigstein2007rds}.

Unfortunately, this standard is not supported by all countries, so seamless mapping is not possible for the time being. In addition, there are problems such as the choice of acronyms for the identifiers, the language and the lack of uniqueness, so that the engineer often performs the mapping manually based on the measured values and previously calculated statistics, since (s)he cannot rely on the predefined identifiers. Classifiers that assign the measured values to the appropriate identifier based on some of their features are suitable for this purpose. Of particular importance are the periodicities or quasi-periodicities occurring within the sensor's signal, which encode certain recurring events within the power plant. We use persistent homology on an embedding of these signals that lies on or dense within an $N$-dimensional torus to encode (quasi-)periodicities. We train neural networks with the raw signal, the zero-dimensional and the one-dimensional homology groups of a filtered toroidal embedding of the signal.

Our work is structured as follows:

\begin{itemize}
  \item[§\ref{primer}] We introduce the theory of persistent homology on triangulable topological spaces. Specifically, we introduce simplicial complexes, filtrations and the associated persistence module, and define some representations of persistent Betti numbers -- and hence the persistence diagram.
  \item[§\ref{embedding}] We derive the assumption to describe a time series as a smooth manifold.
  \item[§\ref{slidingwindowembedding}] We discuss Takens' embedding and the topological and geometrical properties of the sliding window point cloud, which encodes (quasi-)periodicities that can be detected in its persistence diagrams.
  \item[§\ref{heuristics}] We detail the heuristics used to determine an ideal embedding dimension and a time delay, and compute these quantities for our data.
  \item[§\ref{result}] We present the results of examining the data using Betti curves and persistence silhouettes. We evaluate our proposed architecture based on accuracy, F$_1$-score, precision and recall.
  \item[§\ref{discussion}] Finally, we discuss the results and summarize our experiments. We state two issues that arose from our work, particularly with respect to the applicability to other power plants of the same construction type.
\end{itemize}

\section{Primer: Persistent Homology}
\label{primer}
Homology groups are Abelian groups attached to a topological space, counting, in an intuitive sense, the holes of the very same object in a particular dimension. We assume for the persistent homology theory the triangulability of the underlying topological space. This is not essential for homology theories in general, but it is extremely useful for the persistent one. We have no prior knowledge of the underlying space and consequently we must approximate it by some construction on the given points.

Let said topological space contain all points in the data set. We first consider points in general position $\{v_0, \cdots,$ $v_k\} \subset \mathbb{R}^n$, such that the vectors $\{v_1-v_0, \cdots, v_k-v_0\}$ are linearly independent with $k < n$. Thus, the points do not lie on a hyperplane of dimension less than $k$.

Their convex hull is the simplex
\begin{align}
	[v_0, \cdots, v_k] := \left\{ \sum_{i=1}^{k} \lambda_i (v_i-v_0) \; \bigg\vert \; \sum_{i=1}^{k} \lambda_i = 1, \; \lambda_i \geq 0 \right\},
\end{align}
with dimension $k$. The $i$th face of a simplex can be written as
\begin{align}
	d_j [v_0, \cdots, v_k] = [v_0, \cdots, \hat{v}_i, \cdots, v_k],
\end{align}
where $\hat{v}_i$ denotes the removal of the element $v_i$ from the simplex. Note, that the $i$th face is a $(k-1)$-simplex. We call a finite union of such simplices in Euclidean space a simplicial complex $\mathfrak{K}$, if every face of a simplex in $\mathfrak{K}$ is also in $\mathfrak{K}$, and every intersection of two simplices from $\mathfrak{K}$ is either empty or a common face of both. A filtration is a nested sequence
\begin{align}
\mathbb{K}: \quad \emptyset = \mathfrak{K}_0 \subseteq \mathfrak{K}_1 \subseteq \cdots \subseteq \mathfrak{K}_q = \mathfrak{K},
\end{align}
where we assume without restriction of generality that $\mathfrak{K}_i$ = $\mathfrak{K}_q$ for $i \geq q$. The expression $\mathfrak{K}_i$ denotes a one-parameter family of simplicial complexes. The parameter can be interpreted as a particular time point on the filtration and determines how the realization of the simplicial complex looks on the set of points. There are several ways to construct simplicial complexes. For example, the Vietoris-Rips complex $\mathfrak{R}_{i}$ with $i \in \mathbb{R}$, is defined for a set of points as:
\begin{align}
	\mathfrak{R}_{i}(X) := \left\{ U \subseteq X \; \bigg\vert \; \vert\vert x-y\vert\vert \leq i \; \text{for all} \; x,y \in U \right\}.
\end{align}
We are looking for an ideal set of points that contains all the points in our data set. For this, we assume that we can find a suitable set on which we define a topology. This is precisely the assumption of an underlying topological space, which was made at the beginning of the section. If we assume more structure for the topological space we are looking for, such as smooth coordinate maps and a Riemannian metric, we can assume a Riemannian manifold underlying the data.

A Vietoris-Rips complex over a Riemannian manifold (see §\ref{smoothmanifolds}) is homotopy equivalent to the manifold itself for sufficiently small $i$ \cite[§3.5]{hausmann1995vietoris}. Under mild conditions, its homotopy groups have an isomorphism into its respective homology groups by the Hurewicz homomorphism \cite[p.~390,§3]{hatcher2005algebraic,MelodiaL20}. Thus, its homology theory provides an isomorphism to the cohomology theory of a smooth manifold, a suitable description due to the quality of data \cite[§6]{MelodiaL21}.

\subsection{Homology Groups}
The $k$th chain group $C_k$ on $\mathfrak{K}$ is the free Abelian group on the set of $k$-simplices. An element $c \in C_k$ is called $k$-chain and can be written as $c = \sum_{i=1}^{k} \lambda_i \sigma_i$ with $\sigma := [v_0, \cdots, v_k] \in \mathfrak{K}$ with coefficients $\lambda_i$ in any ring. The group becomes a vector space if the coefficients are chosen to be within some field $\mathbb{F}$. The ring of integers $\mathbb{Z}$ modulo a maximal prime ideal -- $\mathbb{Z}/(p\mathbb{Z})$ -- gives us such a field.\footnote{For our work, we use the Mersenne prime $p = 6972593$, because of the efficiency of memory allocation, since it fits into an integer data type and does not cause overflow. One could also use $p=2$, but since about $3\%$ of the data yields homological coefficients other than the one in $\mathbb{Z}_2$, we use the largest possible coefficients field.}

We can study the chain groups on the filtration considering a chain complex, a pair $(C_\star, \partial)$, where $C_\star = \bigoplus_{k \in \mathbb{Z}} C_k$ and $\partial = \bigoplus_{k \in \mathbb{Z}} \partial_k$, with $C_k$ as $\mathbb{F}$-vector space and $\partial_{k+1}: C_{k+1} \rightarrow C_{k}$ as $\mathbb{F}$-linear maps with $\partial_{k} \circ \partial_{k+1} = 0, [v_0,\cdots,v_k] \mapsto \sum_{i=0}^{k} (-1)^i [v_0,\cdots,\hat{v}_i,\cdots,v_k]$. Elements from $\ker \partial_k$ will be called $k$-cycles and to elements from $\text{im} \ \partial_{k+1}$ we will refer to as $k$-boundaries. Thus, each boundary is a cycle. Similarly, cohomology can be defined on simplicial complexes, which is exploited in the implementation of the algorithm for the computation of persistent homology \cite[§4.1]{de2011dualities}. Having the chain complex $(C_\star,\partial)$, the $k$th homology group of chain complexes is defined as a quotient over $\mathbb{F}$ -- or just over some ring -- such that $H_k(C_\star;\mathbb{F}) :=  \ker \partial_k\left(C_\star\right) / \text{im} \  \partial_{k+1}\left(C_\star\right)$. The module is defined as a family of $\mathbb{F}$-vector spaces $V_i$ for a real number $i$ together with $\mathbb{F}$-linear maps $f_{ij}: V_i \rightarrow V_j$, for $i \leq j$, which for a $k \leq i \leq j$ satisfy the equation $f_{kj} = f_{ij} \circ f_{ki}$. We can define persistent homology by considering a family $\{\mathfrak{K}_i\}_{i \in \mathbb{R}}$ of simplicial complexes, with simplicial maps $f_{ij}: \mathfrak{K}_i \rightarrow \mathfrak{K}_j$.

The \emph{persistence module} are the $k$-dimensional homology groups $H_\star(\mathfrak{K}_i; \mathbb{F})$ together with the maps $H_\star(f_{ij}): H_\star(\mathfrak{K}_i; \mathbb{F}) \rightarrow H_\star(\mathfrak{K}_j;\mathbb{F})$ induced by $f_{ij}$.

\subsubsection{Multi-sets} \emph{Persistence diagrams} encode the ranks of a persistence module and are \emph{multi-sets}, a pair consisting of a set and a function $(Z,\varphi)$, where $\varphi: Z \rightarrow \mathbb{R} \cup \{+\infty\}$. For $z \in Z$, $\varphi(z)$ denotes a multiplicity. The union of two multi-sets $(Z,\varphi)$ and $(Z',\varphi')$ is the multi-set $(Z \cup Z', \varphi \cup \varphi')$, with
\begin{equation}
	(\varphi \cup \varphi')(z) =
	\begin{cases}
	\varphi(z), & \text{if} \; z \in Z, z \notin Z',\\
	\varphi'(z), & \text{if} \; z \notin Z, z \in Z',\\
	\varphi(z) + \varphi'(z), & \text{if} \; z \in Z, z \in Z'.\\
	\end{cases}
\end{equation}

\subsection{Persistent Homology}
The persistence module is defined as the collection of all $\mathbb{F}$-vector spaces $V_i$ for a real number $i$ together with the $\mathbb{F}$-linear maps $f_{ij}: V_i \rightarrow V_j$ such that for each pair $i,j$ it holds that $i \leq j$. Moreover, for a $k \leq i \leq j$, it holds that $f_{kj} = f_{ij} \circ f_{ki}$, under the condition that all but finitely many such maps are isomorphisms. Let us choose a filtration of the Vietoris-Rips complex $\{\mathfrak{R}_i(X)\}_{i\in\mathbb{R}}$ together with its simplicial maps $f_{ij}: \mathfrak{R}_i(X) \rightarrow \mathfrak{R}_j(X)$ for each pair $i \leq j$, so that the above mentioned conditions hold. We write the \emph{persistent (simplicial) homology groups} with $\mathbb{F}$-coefficients as persistence module
\begin{align}
	H_\star(\mathfrak{R}_i(X); \mathbb{F}), \ \text{with} \ H_\star(f_{ij}): H_\star(\mathfrak{R}_i(X); \mathbb{F}) \rightarrow H_\star(\mathfrak{R}_j(X); \mathbb{F}).
\end{align}
The persistence diagram is a multi-set of points in $\mathbb{R} \times \left(\mathbb{R} \cup \{+\infty\} \right)$. We consider the case of finite persistence modules, which are represented by persistence diagrams. Such diagrams are realized for a finite and discrete set $I \subset \mathbb{R}$ as real open intervals $\{(b_i,d_j)\}_{i,j\in I}$. The \emph{birth points} $b_i$ and \emph{death points} $d_j$ satisfy for a persistence module $b_i \leq i \leq j < d_j$. The multiplicity of points $(b_i,d_j)$ in the multi-set is equal to $\text{rank}(f_{ij})$. We represent a persistence diagram as a finite multi-set $\mathfrak{P} := \{(\frac{b_i+d_j}{2}, \frac{d_j-b_i}{2})\}_{i,j\in I}$ and $i \leq j$ \cite[§2]{chazal2014stochastic}. We use \emph{persistence landscapes}, a functional representation of  \emph{persistent points} $(\frac{b_i+d_j}{2},\frac{d_j-b_i}{2})$, which are elements of the \emph{persistence diagram} $\mathfrak{P}$, given by the following function \cite[§2.2]{bubenik2015statistical}:
\begin{align}
	\Lambda_i(t) =
	\begin{cases}
		t-b_i, & \text{if} \ t \in [b_i, \frac{b_i+d_j}{2}],\\
		d_j-t, & \text{if} \ t \in (\frac{b_i+d_j}{2},d_j],\\
		0,   & \text{otherwise}.
	\end{cases}
\end{align}

\subsubsection{Silhouettes and Betti Curves}
For points in the multi-set $\mathfrak{P}$ weights $\{w_i = \vert d_j-b_i \vert^p \; \vert \; 0 < p \leq \infty\}_{i,j\in I}$ with $i \leq j$ exist, such that the \emph{$w$-weighted silhouette} of $\mathfrak{P}$ is a representation of a persistence diagram within the vector space of real-valued functions \cite[§2.3]{chazal2014stochastic}:
\begin{align}
	\xi: \mathbb{R} \rightarrow \mathbb{R}, \quad t \mapsto \frac{\sum_{i \in I} w_i \Lambda_i(t)}{\sum_{i\in I} w_i}.
\end{align}
Let the multi-set $\mathfrak{P}$ be a persistence diagram, then we define the \emph{Betti curve} as a function $\beta_{\mathfrak{P}}: \mathbb{R} \rightarrow \mathbb{N}$, whose values on $s \in \mathbb{R}$ are given by the number of points $(b_i,d_j) \in \mathfrak{P}$ -- counted with multiplicity -- such that we satisfy $b_i \leq s \leq d_j$.

We have now introduced persistent homology and the persistence module, and defined persistence diagrams and their representations, which we intend to use for our classifiers. Persistent homology is able to encode (quasi-)periodicities of signals. However, this requires an embedding that encodes the signals as a curve on a compact geometric object, a torus. We discuss this embedding next.
\begin{figure}[!]
	\centering
	\begin{minipage}{.24\linewidth}
		\centering
		\begin{overpic}[trim=70 35 35 70,clip,width=\textwidth]{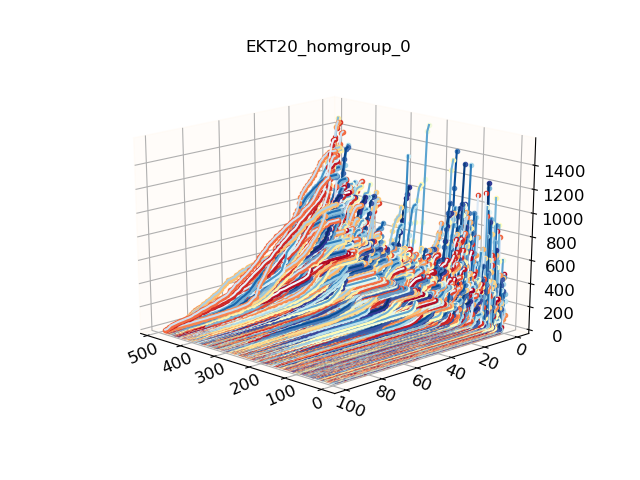}
		 \put (10,50) {\tiny$\circled{1}$}
		 \put (20,4) {\tiny $x$}
		 \put (75,7) {\tiny $y$}
		 \put (92,58) {\tiny $z$}
		\end{overpic}
	\end{minipage}%
	\begin{minipage}{.24\linewidth}
		\centering
		\begin{overpic}[trim=70 35 35 70,clip,width=\textwidth]{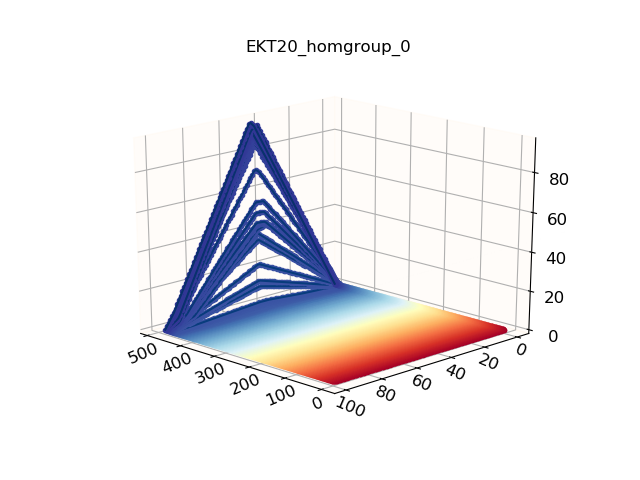}
		 \put (10,50) {\tiny$\circled{2}$}
		 \put (20,4) {\tiny $x$}
		 \put (75,7) {\tiny $y$}
		 \put (92,58) {\tiny $z$}
		\end{overpic}
	\end{minipage}
	\begin{minipage}{.24\linewidth}
		\centering
		\begin{overpic}[trim=70 35 35 70,clip,width=\textwidth]{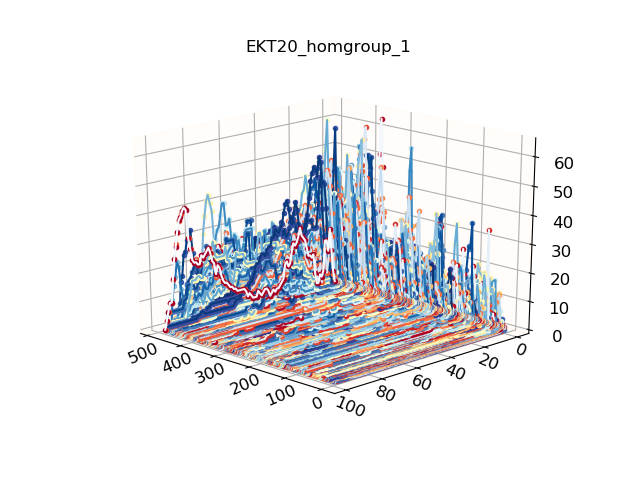}
		 \put (10,50) {\tiny$\circled{3}$}
		 \put (20,4) {\tiny $x$}
		 \put (75,7) {\tiny $y$}
		 \put (92,58) {\tiny $z$}
		\end{overpic}
	\end{minipage}
	\begin{minipage}{.24\linewidth}
		\centering
		\begin{overpic}[trim=70 35 35 70,clip,width=\textwidth]{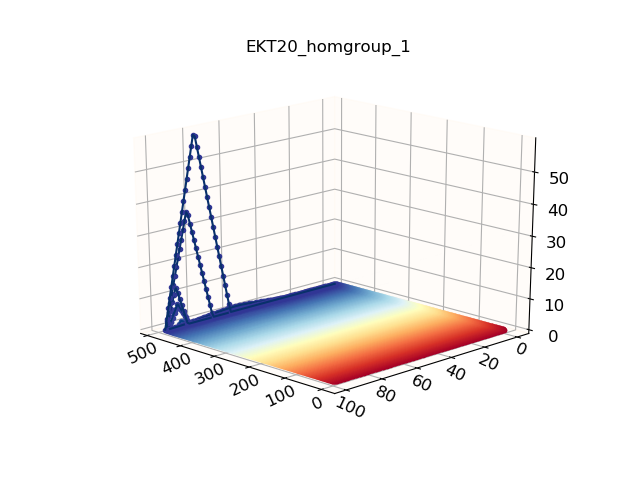}
		 \put (10,50) {\tiny$\circled{4}$}
		 \put (20,4) {\tiny $x$}
		 \put (75,7) {\tiny $y$}
		 \put (92,58) {\tiny $z$}
		\end{overpic}
	\end{minipage}
	\begin{minipage}{.24\linewidth}
		\centering
		\begin{overpic}[trim=70 35 35 70,clip,width=\textwidth]{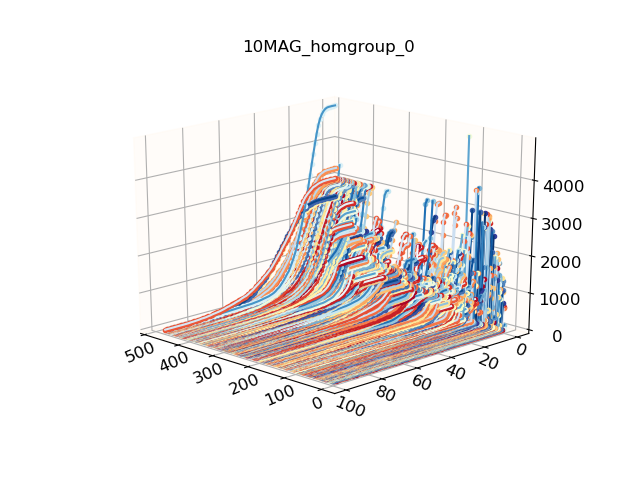}
		 \put (10,50) {\tiny$\circled{5}$}
		 \put (20,4) {\tiny $x$}
		 \put (75,7) {\tiny $y$}
		 \put (92,58) {\tiny $z$}
		\end{overpic}
		\subcaption{$\beta_0$-curves.}
	\end{minipage}%
	\begin{minipage}{.24\linewidth}
		\centering
		\begin{overpic}[trim=70 35 35 70,clip,width=\textwidth]{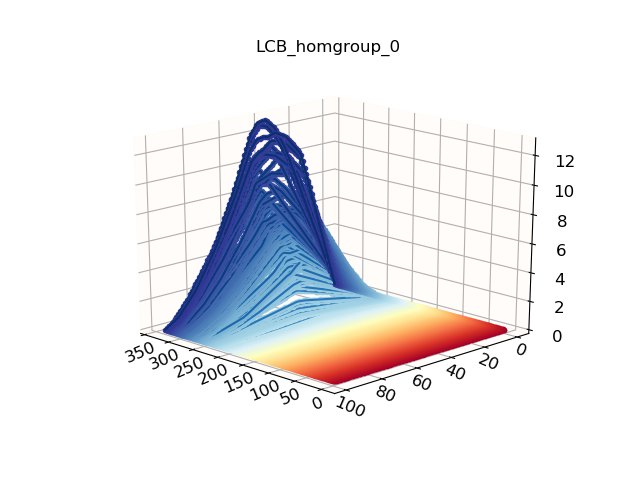}
		 \put (10,50) {\tiny$\circled{6}$}
		 \put (20,4) {\tiny $x$}
		 \put (75,7) {\tiny $y$}
		 \put (92,58) {\tiny $z$}
		\end{overpic}
		\subcaption{$\beta_0$-silhouettes.}
	\end{minipage}
	\begin{minipage}{.24\linewidth}
		\centering
		\begin{overpic}[trim=70 35 35 70,clip,width=\textwidth]{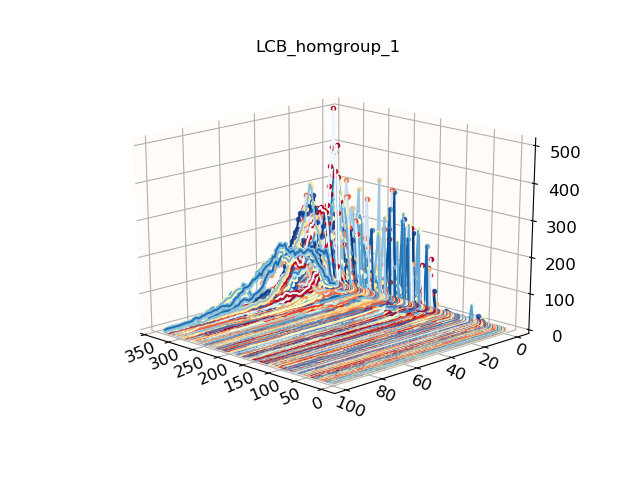}
		 \put (10,50) {\tiny$\circled{7}$}
		 \put (20,4) {\tiny $x$}
		 \put (75,7) {\tiny $y$}
		 \put (92,62) {\tiny $z$}
		\end{overpic}
		\subcaption{$\beta_1$-curves.}
	\end{minipage}
	\begin{minipage}{.24\linewidth}
		\centering
		\begin{overpic}[trim=70 35 35 70,clip,width=\textwidth]{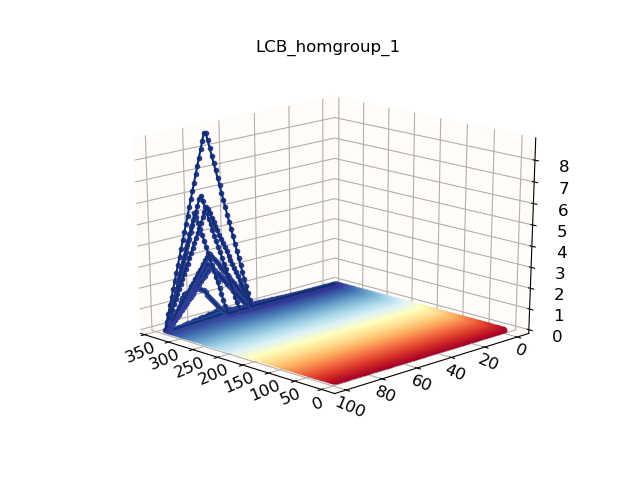}
		 \put (10,50) {\tiny$\circled{8}$}
		 \put (20,4) {\tiny $x$}
		 \put (75,7) {\tiny $y$}
		 \put (92,58) {\tiny $z$}
		\end{overpic}
	\subcaption{$\beta_1$-silhouettes.}
	\end{minipage}
	\caption{Betti curves and silhouettes of power plant signals. The signals are ordered by decreasing persistence entropy. The $x$-axis displays the resulting order. The $y$-axis indicates the parameterization of the simplicial complex normalized to $[0,100]$ -- for visualization purposes. The $z$-axis indicates the number of representatives of the zeroth/first homology group. ($1$-$4$) are zeroth/first Betti curves and persistence silhouettes from the heating medium system of a combined cycle gas turbine power plant with one steam and two gas turbines. ($5$-$8$) show the main condensate pumping plant. The sample size per signal varies between $10^3$ and up to $10^5$.}
	\label{persrepr}
\end{figure}
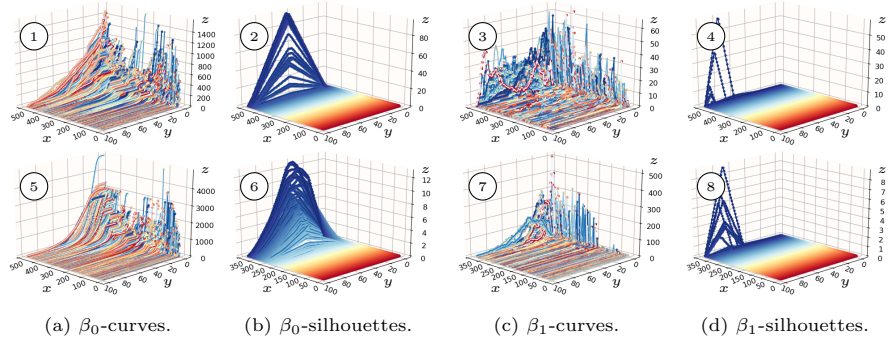

\section{Time Series Embedding}
\label{embedding}
For the embedding of the time series, which is itself a subset of some topological space $T$, we call a mapping $f: T \rightarrow R$ an embedding in a topological space $R$ if $f$ is a homeomorphism from $T$ to the subspace $f(T)$ of its image. Thus $f$ is said to be continuous and injective such that every open set $O \subseteq T$ is open again as an image $f(O) \subseteq f(T)$. Two reasons lead us to the embedding of a time series:
\begin{enumerate}[noitemsep]
	\item The stability of the persistence diagrams we want to compute is guaranteed only for tame functions \cite[§2]{cohen2007stability}. The function $f$ is tame if it is continuous, all sub-level-sets have homology groups of finite rank, and there are finitely many critical values at which these ranks change.
	\item We want to relate homology groups with periodicity. This is done naturally by embedding on a compact geometric object, the $N$-torus.
\end{enumerate}
A time series is a, possibly finite, strictly totally ordered sequence $(t_i) := \{t_i\}_{i=0}^{n}$ of real numbers and corresponds to the measured quantity by a sensor.\footnote{Usually, arbitrary $n$-tuples are written as $(t_i)_{i=0}^{n}$, but we denote explicitly strictly totally ordered ones. We introduce here non-standard notation to ensure readability.} It can be constructed choosing a point $t_0$ and a time step size $s$, such that $t_i = t_0 + s \cdot i$. Taking the sequence $(t_i) \subset T \subset \mathbb{R}$ and a function $f: T \rightarrow \mathbb{R}$ we get another time series $f \circ (t_i) = (f(t_i))$. The total order for a single sequence gives a one-dimensional manifold with smooth structure. We motivate this choice rigorously.

\subsection{Polynomial Approximation}
Let $\mathcal{T} := \{T_j\}_{j = 0}^{m}$ be a finite family of sets and $T_j := \{(t_i) \; \vert \; t_i \in [a,b], i \in \{1, \ldots, n\}\}$ be a finite strictly totally ordered set of points. Our domain is a compact Hausdorff space since every point in $[a,b]$ can be obtained by the intersection of its closed neighborhoods. We can even obtain a $k$-times differentiable function $f:  M \rightarrow \mathbb{R} \in \mathcal{C}^k(\mathbb{R})$ for a compact Hausdorff space $M$ and for any $k \in \mathbb{Z}$, which in turn can be approximated arbitrarily exactly by a polynomial function $p(t)$.

Thus, we can choose polynomial functions as coordinate functions to describe our time series as a smooth manifold, without restriction of generality.

\begin{theorem}
  For a strictly totally ordered sequence $T_j := \{(t_i) \; \vert \; t_i \in [a,b], i \in \{1, \ldots, n\}\}$ there exists a well defined polynomial function $p: M \rightarrow \mathbb{R}$ on a closed interval $M \subset \mathbb{R}$ -- approximating $T_j$ arbitrarily well.
\end{theorem}
\begin{proof}
  To see this, we use Heine-Borel's theorem, which says that any closed interval of the real line is compact \cite[§9.1.24]{tao2006analysis}. We choose $M = [a,b]$. Let $C \subset [a,b]$ be compact, then $f(C)$ is compact in $\mathbb{R}$. Now, we take a sequence $(t_i)$ in the range set $f(C)$. Let $(h_i)$ be a sequence in the domain of $f$. Then, let for $(h_i) \subseteq f(C)$ and every $i \in \mathbb{N}$ be at least one $t_i \in C$ with $f(t_i) = h_i$. Thus, $(t_i) \subseteq C$ is a sequence. Since $C$ is compact, there exists a convergent sub-sequence $(t_{i_{j}})$ whose limit for $j \rightarrow \infty$ is also contained in $C$. Suppose that $f$ is continuous in $[a,b]$, hence, continuous in $t_i \in [a,b]$. Given the fact that $(t_{i_{j}}) \rightarrow t_i$, we conclude that $(h_{i_{j}}) \rightarrow f(t_i)$. Due to the fact that $t_i$ is in $C$, we have that $f(t_i) \in f(C)$, and thus $f(C)$ is compact by continuity. By \cite[§21]{stone1948generalized}, there exists a polynomial function $p(t_i)$ to approximate $f(t_i)$ with arbitrary error $|f(t_i)-p(t_i)| < \epsilon$ on compact Hausdorff spaces. $\qed$
\end{proof}

\subsection{Smooth Manifold Construction}
\label{smoothmanifolds}
As $p(t)$ is smooth, we can use an argument for smooth functions to yield a smooth manifold and its atlas. Roughly speaking a manifold of dimension $n$ is a topological space locally homeomorphic to $\mathbb{R}^n$. A requirement that provides a manifold would be that $\mathbb{M} \subset \mathbb{R}^n$ and for each $t_i \in \mathbb{M}$ there exists an open ball $B_\epsilon(t_i) = \{m \in \mathbb{M} \; \vert \; d(t_i,m) < \epsilon\}$ given the Euclidean metric $d$ and a smooth function, with a smooth diffeomorphism $\phi_{t_i}:B_\epsilon(t_i) \rightarrow \{z \in \mathbb{R}^n \; | \; ||z|| < 1\}$ \cite[§1]{milnor1997topology}. For any smooth function $f: \mathbb{R} \rightarrow \mathbb{R}$ its graph $\mathcal{G} f := \{(t_i,f(t_i)) \; \vert \; t_i \in \mathbb{R}\}$ is a smooth manifold diffeomorphic to $\mathbb{R}$ with an embedding into $\mathbb{R}^2$, given by inclusion. This also holds for functions, that are smooth on a compact domain. In this example, the polynomial is the special case of such a smooth function. An atlas is given by the one map $\varphi: \mathcal{G} f \rightarrow \mathbb{R}, \varphi(t_i,f(t_i)) \mapsto t_i$. Thus, $f$ is a transition map from $\mathbb{R}$ into $\mathcal{G} f$ with $t_i \mapsto (t_i,f(t_i))$. This means that we can assume without restrictions that the manifold underlying our data is smooth.

The graph of the polynomial is a connected topological space, thus the manifold is also connected. In this construction also the higher homology groups are trivial, as $\mathcal{G}f \cong \mathbb{R}$ are isomorphic as manifolds.

But we want to make the homology groups utilizable to detect patterns in the form of (quasi-)periodicities in sequences. Vice versa, their analysis can provide information about outliers. We discuss the intended embedding, which establishes a connection between periodic signals and non-trivial higher homology groups by embedding the underlying $1$-manifold in an $N$-torus.

\section{Sliding-window Embedding}
\label{slidingwindowembedding}
For the analysis of time series the expected behavior of the frequencies and amplitudes are of importance, thus, the behavior of a physical entity measured by sensors over some time. Subsets of the time series can be (quasi-)periodically recurring or single events. We take advantage of this fact and use Takens' embedding to make this behavior topologically and geometrically explicit.

\subsection{Takens' Embedding}
\label{takens}
We need a representation of the signal in a space that makes (quasi-)periodicities and outliers visible. Mathematically, we obtain an embedding into a manifold $M$ using a map $\varphi:\mathbb{M} \times I \rightarrow \mathbb{M}$, with $\mathbb{M} \subseteq \mathbb{R}^d$. We write the embedded sequence for a function $f: \mathbb{M} \rightarrow \mathbb{R}$ as $(f(\varphi(t_i,l)))$ for each $t_i \in \mathbb{M}, l \in I$. We embed the time series into a compact manifold whose homology groups as well as its  dimension can be inferred.

The dimension is a useful invariant that can be used to determine a variety of machine learning algorithm parameters, such as the maximum number of neurons required in a neural network layer.

We thus interpret the time series as a smooth dynamical system, a pair $(\mathbb{M}, \Psi)$ where $\mathbb{M}$ denotes a smooth manifold and $\Psi: \mathbb{M} \times \mathbb{R} \rightarrow \mathbb{M}$ is a flow such that $\Psi(t_i,0) = t_i$ and $\Psi(\Psi(t_i,r),k) = \Psi(t_i,r+k)$ for all $t_i \in \mathbb{M}$ and $r,k \in \mathbb{R}$.

Takens' embedding theorem makes this clear if we choose $\mathbb{M}$ as a compact Riemannian manifold:

Let $\tau \in \mathbb{R}_+\setminus\{0\}$ and $M \geq 2 \dim \mathbb{M}$ be an integer. Furthermore, if $\Psi \in \mathcal{C}^{\infty}(\mathbb{M} \times \mathbb{R},\mathbb{M})$ is a function and $F \in \mathcal{C}^{\infty}(\mathbb{M},\mathbb{M})$ is generic, then $\psi_{t_i}: \mathbb{M} \rightarrow \mathbb{R}$, $r \mapsto F \circ \Psi(t_i,r)$ is called delay map. Thus, $\psi: \mathbb{M} \rightarrow \mathbb{R}^{M+1}$, $t_i \mapsto$ $(\psi_{t_i}(0),$$\psi_{t_i}(\tau)$,$\psi_{t_i}(2\tau)$, $\cdots$, $\psi_{t_i}(M\tau))$ is the desired smooth embedding \cite[§2]{takens1981detecting,packard1980geometry}.

Let $f: [a,b] \rightarrow \mathbb{R}$ be a function. Further, let $M$ be the \emph{embedding dimension} and $\tau$ be called \emph{time delay}. The \emph{sliding-window embedding} of $f: \mathbb{M} \cong \mathbb{R} \rightarrow \mathbb{R}$ into $\mathbb{R}^{M+1}$ can then be written as \cite[§2]{PereaH15}:
\begin{equation}
\label{slidingwindow}
\text{SW}_{M,\tau} f(t_i) = \left[ f(t_i), f(t_i+\tau), \cdots, f(t_i+M \tau) \right]^{\top}.
\end{equation}
The product $M \tau$ is called the \emph{window size} of the sliding-window embedding. For different values of $t_i \in T := [a,b] \subset \mathbb{R}$ we get a set of points that we call \emph{sliding-window point cloud associated with $T$} \cite[§3.2]{perea2019topological}:
\begin{equation}
\label{slidingwindowpointcloud}
\mathbb{S}\mathbb{W}_{M,\tau} f := \left\{ \text{SW}_{M,\tau} f(t_i) \; \vert \; t_i \in T\right\}.
\end{equation}

We have anticipated an embedding which is visibly smooth, but whose geometry and topology we do not know yet. To make the geometry and topology of the associated sliding window point cloud directly exploitable, we treat the periodic and quasi-periodic cases next.

\subsection{Periodic Signals}
\label{periodic}
A $2\pi$-periodic signal can be expressed as a sum of cosine or sine functions together with its Fourier coefficients. A function $f: [t_i,t_i+2\pi] \rightarrow \mathbb{R}$ is $2\pi$-periodic if $f(t_i+2\pi) = f(t_i)$. We pick up the result that the chosen embedding for $L$-periodic functions is dense within a torus. A function $f$ is $L$-periodic on $[0,2\pi]$, iff $f(t_i + \frac{2\pi}{L}) = f(t_i)$ for all $t_i \in \mathbb{R}$.\footnote{Recall, if our signals were $L$-periodic for some natural number $L$, we could, by definition, treat them as functions on some torus $\mathbb{T} \cong \mathbb{R}/(L\mathbb{Z})$.} Under our main assumption, point-wise convergence is also guaranteed since the approximation is allowed to be smooth.

We use a theorem proven by Perea to understand the topological and geometric structure of the \emph{sliding window point cloud} \cite[§5.6]{PereaH15}.

First, we define a centering map $Z:\mathbb{R}^{M+1} \rightarrow \mathbb{R}^{M+1}$,
\begin{align}
	Z(\mathbf{x}) = \mathbf{x} - \frac{\langle \mathbf{x},\mathbf{1} \rangle}{\vert\vert \mathbf{1} \vert\vert^2}, \quad \text{where} \quad \mathbf{1} = [1, \ldots, 1]^{\top} \in \mathbb{R}^{M+1},
\end{align}
and we use the sliding window embedding on the truncated Fourier transform,
\begin{align}
\text{SW}_{M,\tau} S_{N} f(t_i) = \sum_{n=0}^{N} \cos(nt_i) (a_n \mathbf{u}_n + b_n \mathbf{v}_n) + \sin(nt_i)(b_n\mathbf{u}_n - a_n \mathbf{v}_n),
\end{align}
where $S_{N} f(t_i)$ is the \emph{$N$-truncated Fourier series expansion} of $f$ with some remainder, such that $f(t_i) = S_{N}f(t_i) + R_N f(t_i)$, $j^2 = -1$ and
\begin{align*}
S_{N} f(t_i) &= \sum\limits_{n=0}^{N} a_n \cos(nt_i) + b_n \sin(nt_i) = \sum\limits_{n=-N}^{N} \hat{f}(n) e^{jnt_i},\\ \quad
\hat{f}(n) &= \begin{cases}
\frac{1}{2}a_n - \frac{j}{2}b_n, & \text{if} \; n > 0,\\
\frac{1}{2}a_{-n} - \frac{j}{2}b_{-n}, & \text{if} \; n < 0,\\
a_0, & \text{if} \; n = 0.
\end{cases}
\end{align*}
If we take $f$ to be $L$-periodic, with $L(M+1)\tau=2\pi$, then we yield $\text{SW}_{M,\tau} S_N f(t_i)$ $=$ $\hat{f}(0) \cdot \mathbf{1} + Z(\text{SW}_{M,\tau} S_N f(t_i))$. For its norm we obtain $|| Z(\text{SW}_{M,\tau} S_N f(t_i)) ||$ $=$ $\sqrt{M+1}( ||S_N f ||^2_2 - \hat{f}(0)^2)^{1/2}$. Constructing the orthonormal set $\{ \widetilde{\mathbf{x}}_n,\widetilde{\mathbf{y}}_n \in \mathbb{R}^{M+1} \; \vert \; 1 \le n \leq N, n \equiv 0 \mod L \}$ gives, as proven in \cite[§5.6]{PereaH15}:
\begin{align}
	\varphi_{\tau}(t_i) &:= \frac{\sqrt{M+1}\left( ||S_N f ||^2_2 - \hat{f}(0)^2\right)^{\frac{1}{2}}}{||\sqrt{M+1}\left( ||S_N f ||^2_2 - \hat{f}(0)^2\right)^{\frac{1}{2}}||} \\
  &= \sum\limits_{\substack{n=1 \\ n \equiv 0 \mod L}}^{N} \frac{2|\hat{f}(n)|}{\sqrt{||S_N f||^2_2 - \hat{f}(0)^2}}\left(\cos(nt_i)\widetilde{\mathbf{x}}_n + \sin(nt_i)\widetilde{\mathbf{y}}_n\right)\\
  \label{slidingwindowr}
  &= \sum\limits_{\substack{n=1 \\ n \equiv 0 \mod L}}^{N} \widetilde{r}_n\left(\cos(nt_i)\widetilde{\mathbf{x}}_n + \sin(nt_i)\widetilde{\mathbf{y}}_n\right), \quad \text{with} \quad \sum_{n=1}^{N} \widetilde{r}_n^2 = 1.
\end{align}
As studied in \cite[§5.6]{PereaH15}, a clear geometric picture of the centered and normalized sliding window point cloud arises for $S_N f$, since if we consider $S^1(r) \subset \mathbb{C}$ as a circle centered around zero with radius $r$ in the complex plane, then $t_i \rightarrow \varphi_\tau(t_i)$ is a curve on an $N$-torus $\mathbb{T} = S^1(\widetilde{r}_1) \times \cdots \times S^1(\widetilde{r}_N)$, as shown in Eq. \ref{slidingwindowr}.

Thus, we can use the zeroth persistent homology group - as a Euclidean metric dependent quantity between the connected components of Takens' embedding of the signal samples - and the first persistent homology group as features for our classifier to describe the torus on which the data lie. The dependence on the Euclidean metric is due to the choice of the embedding space of the simplicial complex. Another Riemannian manifold would also be conceivable, but we use the ordinary Riemannian metric of $\mathbb{R}^n$.

For the $N$-torus the Betti numbers are computed as $\beta_k = {N\choose k}$ for the $k$th persistent homology group. It follows that $\text{rank} \ H_1(\varphi_{\tau};\mathbb{F})$ corresponds to the number of $1$-spheres forming the product space $\mathbb{T}$ \cite{MelodiaL21}. Thus, the underlying manifold can be determined.

\subsection{Quasi-periodic Signals}
\label{quasiperiodic}
For the second case, we are interested in functions $f:[a,b] \rightarrow \mathbb{C}$, which are quasi-periodic signals of the form $f(t_i) = \sum_{l=0}^{L} \lambda_l e^{j\omega_lt_i}$, where $n \in \mathbb{N}$, $\lambda_l$ are non-negative complex numbers, and $\omega_l$ are incommensurate non-negative reals. As such, they are functions which also have incommensurate periods. Such signals naturally occur in the context of power plants.

Perea proved \cite[§2]{perea2016persistent}, that for points
\begin{align}
p_f(t_i) = (\lambda_0 e^{j\omega_0 t_i}, \lambda_1 e^{j\omega_1 t_i}, \cdots, \lambda_L e^{j\omega_L t_i})
\end{align}
and for $\lambda_l \in \mathbb{C}$, the generated set $T_f = \{p_f(t_i) \; \vert \; t_i \in \mathbb{Z}\}$ is dense within an $(N+1)$-torus $\mathbb{T}^{N+1} = S^1_{\lambda_0} \times S^1_{\lambda_1} \times \cdots \times S^1_{\lambda_N}$, by Kronecker's Theorem \cite{kronecker1884naherungsweise}. More succinct: if $0 < \tau < 2\pi / \max(\omega_l)$, then $\Omega_f$ has full rank. Moreover, if $M \geq N$, then the sliding window point cloud is given by
\begin{align}
\mathbb{S}\mathbb{W}_{M,\tau} f := \left\{ \text{SW}_{M,\tau} f(t_i) \; \vert \; t_i \in \mathbb{Z}\right\},
\end{align}
and is dense within a space homeomorphic to $\mathbb{T}^{N+1}$  \cite[§2.1]{perea2016persistent}.

We address a rather trivial observation that makes all reasoning amenable to computations in the real number field: The complex plane has an isomorphism of sets given by the bijection $\wp: \mathbb{R} \times \mathbb{R} \rightarrow \mathbb{C}, (t_i,s) \rightarrow t_i + js$, with $j^2 = -1$, where $\times$ denotes the direct product of sets. The vector space structure can be defined on both, $\mathbb{R} \times \mathbb{R}$ as well as on $\mathbb{C}$, with basis $B_{\mathbb{R} \times \mathbb{R}} = \{(1,0),(0,1)\}$ and $B_{\mathbb{C}} = \{(1,j)\}$. In the category of rings and fields the two objects differ, but for our purposes we use the bijection $\wp$ as an isometry between $\mathbb{C}$ and $\mathbb{R}^2$ -- allowing us to apply the above arguments to real vector spaces. Since these $\mathbb{R}$-vector spaces are both of dimension two, they are isomorphic as $\mathbb{R}$-modules.

\section{Heuristic Choice of Parameters}
\label{heuristics}
Some of the hyper-parameters have to be estimated in advance by heuristic approaches. We describe them next and determine the ideal embedding dimension and time-delay of our particular data set.

\subsection{Time Delay $\tau$}
\label{timedelay}
We determine an optimal value for the time delay creating a partition for each time series $T_j$. For this approach, we write the interval $[t_{1}^{\min},t_{n}^{\max}] \subset T_j$ as the interval between the smallest and largest value for $t_i \in T_j$ of a discrete multi-set of time series values, i.e., $i \in \mathbb{Z}$. Accordingly, $t_{1}^{\min}$ is a global minimum and $t_{n}^{\max}$ is a global maximum. The interval is split into a partition of $s$-size:
\begin{align}
	\mathcal{P}_{T_j} :&= \left\{ \underbrace{[t_{1}^{\min}, t_{s}^{\min}]}_{:=[1]}, \underbrace{[t_{s + 1}, t_{2s + 1}]}_{:=[2]}, \underbrace{[t_{2s + 2}, t_{3s + 2}]}_{:=[3]}, \cdots, \underbrace{[t_{n-s}^{\max}, t_{n}^{\max}]}_{:=[n/s]} \right\}.
\end{align}

We choose $s$ as the smallest divisor of $n$ for the time steps whose particular partitioning contains the fewest turning points. We justify this choice by arguing that a large number of elements are defined in $\mathcal{P}_{T_j}$, but each interval most likely contains a single period. Hereafter we will refer to all intervals with $k \in \{1,2, \cdots, n/s\}$ as \emph{bins}.

Let $\mathsf{P}([k]) := \mathsf{P}(t_i \in [k])$ denote the probability that $t_i$ is contained in the $k$th bin. Let $\mathsf{P}([k],[l]) : = \mathsf{P}(t_i \in [k], t_{i+\tau} \in [l])$ denote the probability that $t_i$ is contained in the $k$th bin, while $t_{i+\tau}$ is contained in the $l$th bin, with $k \neq l$. The mutual information $\mathsf{I}([k],[l] \vert \tau)$ for these two probability distributions is then computable as follows:
\begin{align}
	\mathsf{I}([k],[l] \vert \tau) = - \sum_{k=1}^{n/s} \sum_{l=1}^{n/s} \mathsf{P}([k],[l] \vert \tau) \; \text{ld} \; \frac{\mathsf{P}([k],[l] \vert \tau)}{\mathsf{P}([k] \vert \tau)\mathsf{P}([l] \vert \tau)}.
\end{align}
To gain intuition, recall the entropy for a finite alphabet $T_j = \{t_1,t_2,\ldots,t_n\}$ as $\mathsf{H}(T_j) = - \sum_{t_i \in T_j} \mathsf{P}(t_i) \; \text{ld} \; \mathsf{P}(t_i)$, which gives the mean value of bits for each element to encode it uniquely. Hence, the mutual information indicates how many bits are needed to generate a second random variable, given one at hand, with both functions running over the alphabet.

The optimal period $\tau$ is obtained for $\min_\tau \mathsf{I}([k],[l] \vert \tau)$, since we get the most information when we add another value $t_{i+\tau}$ to this particular subset. We minimize $\mathsf{I}([k],[l] \vert \tau)$. We get as a result $\tau=1$ for over $98\%$ of the data. The size of one bin ranges from $463$ to $624$. Accordingly, we choose a sample size of $500$.

\subsection{Embedding Dimension $M$}
\label{embeddingdimension}
We examined about $18 \cdot 10^3$ different signals coming from four different combined cycle gas turbine power plants with a total of two gas turbines, two boilers for steam generation and one steam turbine of the same construction type.

A constant embedding dimension must be specified for our model, which isn't optimal for the individual signal. However, since the toroidal embedding is also to be performed for the preprocessing of a new signal, a constant embedding dimension is a prerequisite. To determine a suitable $M$, we use the \emph{false nearest neighbor} algorithm. We apply the assumption from §\ref{embedding} that the embedding of a deterministic system in higher dimensions is smooth, relying on the generality of Takens' embedding. Thus, points that are close as measured values on the graph of their function will be close in their embedding with respect to the induced Euclidean norm in data space. For two points $t_i,t_k \in T_j$ there is an $\epsilon^{(ik)}$ such that:
\begin{wrapfigure}[13]{r}{2.7cm}
  \begin{tikzpicture}[scale=0.7]
	\begin{axis}[
	     width  = 5cm,
	     hide y axis,
	     axis x line*=bottom,
	     height = 4.95cm,
	     bar width=15pt,
	     symbolic x coords={1, 2, 3, 4},
	     xticklabels={M=1, M=2, M=3, M=4, M=5},
	     nodes near coords,
	     ymin=0
	     ]
	     \addplot[ybar, fill=darkblue] coordinates {
	          (1,4345)
	          (2,2594)
	          (3,3877)
	          (4,7347)
	     };
	\end{axis}
	\end{tikzpicture}
	\caption{\footnotesize{Cardinality of the set of time series with their optimal embedding dimension $M$.}}
  \label{countsembeddingM}
\end{wrapfigure}
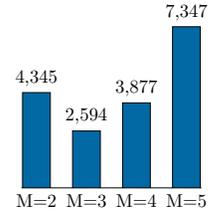

\begin{align}
\epsilon^{(ik)} = \frac{|t_{i+M\tau} - t_{k+M\tau}|}{\left(\frac{1}{n} \sum_{l=1}^{n}(t_l-\frac{1}{n}\sum_{l=1}^{n} t_l )^2\right)^{\frac{1}{2}}} > \epsilon.
\label{falseneighbors}
\end{align}

The $L^1$-norm is chosen for the numerator because it is more robust to outliers in the sense that the values are not higher powers of the absolute value. Whenever the inequality in Eq. \ref{falseneighbors} is satisfied, $t_i$ and $t_k$ are called \emph{false nearest neighbors}. As suggested by the authors of the original algorithm \cite{kennel2002false}, the difference $\vert t_{i+M\tau} - t_{k+M\tau} \vert$ would in probability be as large as $\sqrt{2}\sigma_{T_j}$, if the data were white noise of standard deviation $\sigma$. Thus, we set $\epsilon = \sqrt{2} \sigma_{T_j}$, for every time series $T_j \in \mathcal{T}$. The optimal embedding dimension is the optimization problem that minimizes the number of such false neighbors. We choose $M = 5$, since this is the ideal embedding for the majority of the time series according to Fig. \ref{countsembeddingM}, as we can always immerse all other signals with lower optimal embedding dimension.

\section{Results}
\label{result}
To compute simplicial complexes, filtrations and persistence representations, we used \texttt{GUDHI v.3.4.1} \cite{maria2014gudhi} and \texttt{giotto-tda v.0.4.0} \cite{tauzin2021giotto}. Neural networks were written in \texttt{Keras v.2.4.0} \cite{chollet2015} and trained using \texttt{Tensorflow v.2.4.0} \cite{tensorflow2015-whitepaper} as backend on \texttt{NVIDIA Quadro RTX 4000} graphics cards. \texttt{cuDNN v.8.4.0} \cite{chetlur2014cudnn} is used to enable graphics acceleration for LSTMs.
\definecolor{swblue}{RGB}{19, 82, 138}
\definecolor{lstm}{RGB}{0, 51, 102}
\definecolor{swred}{RGB}{253, 94, 83}
\definecolor{hom0}{RGB}{206, 79, 74}
\definecolor{hom1}{RGB}{248, 201, 60}
\definecolor{data}{RGB}{50, 40, 66}

\begin{figure}[!b]
  \begin{tikzpicture}
    \draw[draw=black, fill=black] (0,0) rectangle (2.5,0.4);
    \node at (1.25,0.175) {\small\color{white}\texttt{Shape: (500,1)}};
    \draw[draw=black, fill=black] (2.75,0) rectangle (5.25,0.4);
    \node at (4,0.175) {\small\color{white}\texttt{Shape: (500,1)}};
    \draw[draw=black, fill=black] (5.5,0) rectangle (8,0.4);
    \node at (6.75,0.175) {\small\color{white}\texttt{Shape: (500,1)}};

    \draw[draw=black, fill=swred] (0,-0.2) rectangle (2.5,-0.6);
    \draw[draw=black, fill=swred] ((2.75,-0.2) rectangle (5.25,-0.6);
    \draw[draw=black, fill=swred] (5.5,-0.2) rectangle (8,-0.6);
    \draw[draw=black, fill=swblue] (0,-0.59) rectangle (2.5,-0.99);
    \node at (1.25,-0.4) {\small\color{white}$\texttt{Conv1D\_1\_01}$};
    \node at (4,-0.4) {\small\color{white}$\texttt{Conv1D\_2\_01}$};
    \node at (6.75,-0.4) {\small\color{white}$\texttt{Conv1D\_3\_01}$};
    \node at (1.25,-0.8) {\small\color{white}\texttt{BatchNorm}};

    \draw[draw=black, fill=swred] (0,-1.2) rectangle (2.5,-1.6);
    \draw[draw=black, fill=swred] (2.75,-1.2) rectangle (5.25,-1.6);
    \draw[draw=black, fill=swred] (5.5,-1.2) rectangle (8,-1.6);
    \draw[draw=black, fill=swblue] (0,-1.6) rectangle (2.5,-1.99);
    \node at (1.25,-1.4) {\small\color{white}$\texttt{Conv1D\_1\_02}$};
    \node at (4,-1.4) {\small\color{white}$\texttt{Conv1D\_2\_02}$};
    \node at (6.75,-1.4) {\small\color{white}$\texttt{Conv1D\_3\_02}$};
    \node at (1.25,-1.8) {\small\color{white}\texttt{BatchNorm}};
    \node at (1.25, -2.4) {$\vdots$};
    \draw[draw=black, fill=swred] (0,-3.0) rectangle (2.5,-3.4);
    \draw[draw=black, fill=swred] ((2.75,-3.0) rectangle (5.25,-3.4);
    \draw[draw=black, fill=swred] (5.5,-3.0) rectangle (8,-3.4);
    \draw[draw=black, fill=swblue] (0,-3.39) rectangle (2.5,-3.8);
    \node at (1.25,-3.2) {\small\color{white}$\texttt{Conv1D\_1\_42}$};
    \node at (4,-3.2) {\small\color{white}$\texttt{Conv1D\_2\_42}$};
    \node at (6.75,-3.2) {\small\color{white}$\texttt{Conv1D\_3\_42}$};
    \node at (1.25,-3.6) {\small\color{white}\texttt{BatchNorm}};
    \draw[black] (1.25,0)--(1.25,-0.2);
    \draw[black] (1.25,-1)--(1.25,-1.2);
    \draw[black] (1.25,-2.0)--(1.25,-2.25);
    \draw[black, ->] (1.25,-2.75)--(1.25,-3.0);
    \draw [black,->] (2.2,-0.4) to [out=20,in=90] (2.2,-1.45);
    \node[scale=0.5] at (2.32,-0.8) {\small$\circled{$+$}$};
    \draw [black,->] (2.2,-1.5) to [out=20,in=90] (2.2,-2.3);
    \draw [black,->] (2.2,-2.75) -- (2.2,-3.25);
    \node at (2.2, -2.4) {$\vdots$};
    \draw[black] (4,0)--(4,-0.2);
    \draw[black] (4,-0.6)--(4,-1.2);
    \draw[black] (4,-1.6)--(4,-2.25);
    \draw[black, ->] (4,-2.75)--(4,-3.0);
    \draw [black,->] (4.95,-0.4) to [out=20,in=90] (4.95,-1.45);
    \node[scale=0.5] at (5.07,-0.8) {\small$\circled{$+$}$};
    \draw [black,->] (4.95,-1.5) to [out=20,in=90] (4.95,-2.3);
    \draw [black,->] (4.95,-2.75) -- (4.95,-3.25);
    \node at (4.95, -2.4) {$\vdots$};
    \node at (4, -2.4) {$\vdots$};

    \draw[black] (6.75,0)--(6.75,-0.2);
    \draw[black] (6.75,-0.6)--(6.75,-1.2);
    \draw[black] (6.75,-1.6)--(6.75,-2.25);
    \draw[black, ->] (6.75,-2.75)--(6.75,-3.0);
    \draw [black,->] (7.75,-0.4) to [out=20,in=90] (7.75,-1.45);
    \node[scale=0.5] at (7.87,-0.8) {\small$\circled{$+$}$};
    \draw [black,->] (7.75,-1.5) to [out=20,in=90] (7.75,-2.3);
    \draw [black,->] (7.75,-2.75) -- (7.75,-3.25);
    \node at (7.75, -2.4) {$\vdots$};
    \node at (6.75, -2.4) {$\vdots$};

    \draw[draw=black, fill=lstm] ((2.75,-4.5) rectangle (5.25,-4.9);
    \node at (4,-4.7) {\small\color{white}\texttt{LSTM\_01}};
    \draw[draw=black, fill=lstm] ((2.75,-5.1) rectangle (5.25,-5.5);
    \node at (4,-5.3) {\small\color{white}\texttt{LSTM\_02}};
    \draw[draw=black, fill=lstm] ((2.75,-6.3) rectangle (5.25,-6.7);
    \node at (4,-6.5) {\small\color{white}\texttt{LSTM\_22}};
    \draw[draw=black, fill=black] ((2.75,-6.9) rectangle (5.25,-7.3);
    \node at (4,-7.1) {\small\color{white}\texttt{Classes: (303)}};
    \draw [black,->] (4.95,-4.75) to [out=20,in=90] (4.95,-5.35);
    \draw [black,->] (4.95,-5.4) to [out=20,in=90] (4.95,-5.7);
    \node[scale=0.5] at (5.07,-5) {\small$\circled{$+$}$};
    \node at (4.95, -5.8) {$\vdots$};
    \draw [black,->] (4.95,-6.1) -- (4.95,-6.55);

    \node at (10.3,-1.37) {
      \resizebox{3.5cm}{!}{
      \begin{tabular}{lc}
      \toprule
      \textsc{Conv1D} & \textbf{Parameters} \\ \midrule
      Activation & $\frac{t}{1+e^{-\beta t}}$ \cite{ramachandran2017swish}\\\arrayrulecolor{gray}\hline
      \#-Layer & $42$ \\\arrayrulecolor{gray}\hline
      Filters & $64$\\\arrayrulecolor{gray}\hline
      Kernel-size & $3$\\\arrayrulecolor{gray}\hline
      Padding & Causal \\\arrayrulecolor{gray}\hline
      Kernel init. & Glorot normal \\\arrayrulecolor{gray}\hline
      Bias init. & Zeros \\\arrayrulecolor{gray}\hline
      Residual & $\mathcal{C}^1$ \cite{HauserGJR19}\\\arrayrulecolor{gray}\hline
      L$^1$-regularization & $0.001$ \\\arrayrulecolor{gray}\hline
      L$^2$-regularization & $0.01$ \\\arrayrulecolor{black}
      \bottomrule
    \end{tabular}}
    };
    \node at (10.3,-5) {
      \resizebox{3.5cm}{!}{
      \begin{tabular}{lc}
      \toprule
      \textsc{LSTM} & \textbf{Parameters} \\ \midrule
      Activations & See \cite{hochreiter1997long}\\\arrayrulecolor{gray}\hline
      \#-Layer & $22$ \\\arrayrulecolor{gray}\hline
      Units & $32$\\\arrayrulecolor{gray}\hline
      Kernel init. & Glorot normal \\\arrayrulecolor{gray}\hline
      Bias init. & Zeros \\\arrayrulecolor{gray}\hline
      Residual & $\mathcal{C}^1$ \cite{HauserGJR19}\\\arrayrulecolor{gray}\hline
      L$^1$-regularization & $0.001$ \\\arrayrulecolor{gray}\hline
      L$^2$-regularization & $0.01$ \\\arrayrulecolor{black}
      \bottomrule
    \end{tabular}}
    };
    \draw[black] (1.25,-3.8) to [out=-120,in=90] (4.15,-4.2);
    \draw[black] (6.75,-3.8) to [out=-60,in=90] (3.9,-4.2);
    \draw[black] (6.75,-3.4) to [out=-90,in=110] (6.75,-3.8);
    \draw[black, ->] (4,-3.4)--(4,-4.5);
    \node[scale=0.5] at (4,-4.1) {\huge$\circled{$+$}$};
    \draw[black] (4,-4.9)--(4,-5.1);
    \draw[black] (4,-5.5)--(4,-5.7);
    \node at (4,-5.8) {$\vdots$};
    \draw[black,->] (4,-6.1)--(4,-6.3);
    \draw[black,->] (4,-6.7)--(4,-6.9);
  \end{tikzpicture}
  \caption{An illustration of the derived architecture. We trained on sigmoid focal cross-entropy loss \cite{ross2017focal}, with a batch size of $128$, a learning rate of $10^{-4}$ with stopping patience for $5$ epochs, and a minimum learning rate of $10^{-6}$ for about $6 \cdot 10^3$ epochs. The task was to classify $303 \cdot 10^3$ samples into $303$ classes, with $10^3$ samples per class. We used $10\%$ of the data for validation and did no further tuning of hyper-parameters.}
  \label{architecture}
\end{figure}
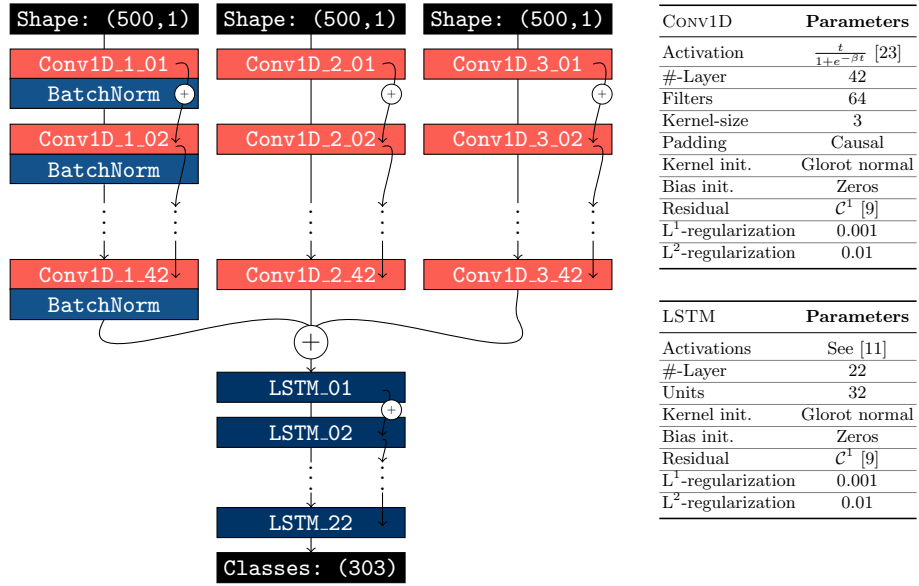

In a first step, we removed all time series from the data that do not have persistent features of the zeroth and first homology group in their persistence diagrams. Thus, we removed all time series with persistence entropy
\begin{align*}
\mathsf{H}\left(\mathfrak{P}_{\mathbb{SW}_{M,\tau}f}\right) &= -\sum_{(d_j,b_i) \in \mathfrak{P}_{\mathbb{SW}_{M,\tau}f}} \frac{d_j-b_i}{\kappa} \; \text{ld} \; \left(\frac{d_j-b_i}{\kappa}\right) \geq 0.98,\\
\text{where} \quad \kappa &= \sum\limits_{(d_j,b_i) \in \mathfrak{P}_{\mathbb{SW}_{M,\tau}f}}d_j-b_i.
\end{align*}
The goal is to get rid of erroneous measurements, as well as (nearly) uniformly distributed signals or constants. We then replaced all \texttt{NaN} values in the time series and Betti curves with their corresponding median, respectively.

\subsection{Persistence Representations}
Each individual colored curve in Fig. \ref{persrepr} (a), (c) represents a recorded signal of a power plant within the respective component over one year.

Fig. \ref{persrepr} (b), (d) show the corresponding persistence silhouettes.

\subsubsection{$\beta_0$-curves} The $\beta_0$-curves in Fig. \ref{persrepr} (a) count the number of representatives of the zeroth homology group at each parameter $y$ of the filtration. Plateaus indicate persistent properties, i.e., connectivity is preserved for some time. The curve for $\beta_0$ is monotonically decreasing. Initially, each measurement counts itself as a connected component -- as it is an embedded point, so the amount of points available is trivially the largest value of each $\beta_0$-curve. Thus, the plot of Betti curves gives information about the number of points per signal. It visualizes the class balance of the data set. We can measure this balance considering the persistence entropy of all $\beta_0$-curves from Fig. \ref{persrepr} (a), restricted to $y=0$, of about $67\%$. For $y \in \{1, \ldots, 20\}$, i.e., the first twenty recorded steps of the filtration, the entropy of the $\beta_0$-curves in Fig. \ref{persrepr} (a) -- of all signals -- varies on average between about $67\%$ and $45\%$, but decreases as $y$ increases. This means that there is a large difference in the number of connected components in the first $20\%$ of the computed filtration. Entropy gives a measure of the distribution of connected components among the signals at some specific parameter $y$ of the filtration. Thus, we can assume that most of the data differs with regard to its density distribution in Euclidean space. Further, Fig. \ref{persrepr} (a) shows $\beta_0$-curves, restricted to $x=500$, which are almost diagonal. This means that connected components are present in large number throughout the filtration for this particular signal.

\subsubsection{$\beta_1$-curves} Similarly, the persistence entropy of the $\beta_1$-curves in Fig. \ref{persrepr} (c), restricted to $y=50$, is about $47\%$. Therefore, the signals differ by the number of cyclic elements. This can also be easily seen by looking at the scale of the $z$-axes in Fig. \ref{persrepr} (a).

Recall that the signals are toroidally embedded. If no connected components are merged, it means that along the filtration we can expect an $S^1$-factor of the torus. In other words, the representatives of the first homology group, which we can see in the diagrams Fig. \ref{persrepr} (c), (d), count the periodic elements within the signal. The toroidal embedding allows us to interpret that certain modes of the original signal lie on different tori. This claim is directly observable in Fig. \ref{persrepr} (c) (3); the diagram shows numerous counts of holes within the embedding per signal, even for higher filtration parameters.

Since the dimension $N$ of the torus corresponds to the rank of its first homology group, we can also read the dimension of the toroidal embedding of a signal from the persistence diagrams and thus from its $\beta_1$-curve in Fig. \ref{persrepr} (c) (d). As $\beta_0$-curves contain the information about connected components and the dimension of the torus is determined by the first homology group, we get information about the object we are embedding into and its geometry. According to Takens' embedding, each period in the signal is mapped to a $1$-sphere of the factor spaces of the $N$-torus, see §\ref{periodic} and §\ref{quasiperiodic}, which is why the $\beta_1$-curves `\emph{encode events}' within the power plant, recorded by the respective sensor.

The Betti curves lie in a Hilbert space, allowing statistics to be computed, and can themselves be resolved to fit the length of the time series sample, making them particularly suitable for the neural network architecture we have chosen.

\subsection{Classification}
\label{classification}
\begin{table}[t]
  \centering
  \begin{tabular}{cccccccc}
  \toprule
  \textbf{OS} & \textbf{F } & \textbf{A } & \textbf{OR} & \textbf{Accuracy} & \textbf{F1} & \textbf{Precision} & \textbf{Recall}\\\midrule
  \multicolumn{8}{c}{\textsc{$\mathcal{C}^0$-ConvNet without topological features similar to Fig. \ref{architecture}:}} \\\midrule
  \cmark & \cmark & \cmark & \cmark & 0.4821 {\tiny$\pm 0.0031$} & 0.5677 {\tiny$\pm 0.0033$} & 0.6912 {\tiny$\pm 0.0029$} & 0.4816 {\tiny$\pm 0.0037$}\\\hdashline
  \color{gray}\cmark & \color{gray}\xmark & \color{gray}\xmark & \color{gray}\xmark & \color{gray}0.7129 {\tiny$\pm 0.0102$} & \color{gray}0.7904 {\tiny$\pm 0.0092$} & \color{gray}0.9010 {\tiny$\pm 0.0097$} & \color{gray}0.7041 {\tiny$\pm 0.0088$} \\\hdashline
  \color{gray}\cmark & \color{gray}\cmark & \color{gray}\xmark & \color{gray}\xmark & \color{gray}0.5691 {\tiny$\pm 0.0037$} & \color{gray}0.6830 {\tiny$\pm 0.0058$} & \color{gray}0.8699 {\tiny$\pm 0.0065$} & \color{gray}0.5622 {\tiny$\pm 0.0052$} \\\hdashline
  \color{gray}\cmark & \color{gray}\cmark & \color{gray}\cmark & \color{gray}\xmark & \color{gray}0.5426 {\tiny$\pm 0.0055$} & \color{gray}0.6681 {\tiny$\pm 0.0036$} & \color{gray}0.8682 {\tiny$\pm 0.0048$} & \color{gray}0.5429 {\tiny$\pm 0.0029$} \\\midrule
  \multicolumn{8}{c}{\textsc{$\mathcal{C}^0$-ConvNet as in Fig. \ref{architecture}:}} \\\midrule
  \cmark & \cmark & \cmark & \cmark & 0.6142 {\tiny$\pm 0.0047$} & 0.6212 {\tiny$\pm 0.0077$} & 0.7681 {\tiny$\pm 0.0082$} & 0.5216 {\tiny$\pm 0.0073$} \\\hdashline
  \textbf{\color{gray}\cmark} & \textbf{\color{gray}\xmark} & \textbf{\color{gray}\xmark} & \textbf{\color{gray}\xmark} & \textbf{\color{gray}0.8316 {\tiny$\pm 0.0121$}} & \textbf{\color{gray}0.8511 {\tiny$\pm 0.0063$}} & \textbf{\color{gray}0.9327 {\tiny$\pm 0.0163$}} & \textbf{\color{gray}0.7827 {\tiny$\pm 0.0039$}} \\\hdashline
  \color{gray}\cmark & \color{gray}\cmark & \color{gray}\xmark & \color{gray}\xmark & \color{gray}0.7024 {\tiny$\pm 0.0091$} & \color{gray}0.7567 {\tiny$\pm 0.0101$} & \color{gray}0.8756 {\tiny$\pm 0.0109$} & \color{gray}0.6663 {\tiny$\pm 0.0094$} \\\hdashline
  \color{gray}\cmark & \color{gray}\cmark & \color{gray}\cmark & \color{gray}\xmark & \color{gray}0.6291 {\tiny$\pm 0.0078$} & \color{gray}0.7376 {\tiny$\pm 0.0065$} & \color{gray}0.8726 {\tiny$\pm 0.0056$} & \color{gray}0.6389 {\tiny$\pm 0.0077$} \\\midrule
  \multicolumn{8}{c}{\textsc{$\mathcal{C}^1$-ConvNet as in Fig. \ref{architecture}:}} \\\midrule
  \textbf{\cmark} & \textbf{\cmark} & \textbf{\cmark} & \textbf{\cmark} & \textbf{0.6383 {\tiny$\pm 0.0085$}} & \textbf{0.6566 {\tiny$\pm 0.0055$}} & \textbf{0.7849 {\tiny$\pm 0.0074$}} & \textbf{0.5597 {\tiny$\pm 0.0076$}} \\\hdashline
  \color{gray}\cmark & \color{gray}\xmark & \color{gray}\xmark & \color{gray}\xmark & \color{gray}0.8221 {\tiny$\pm 0.0028$} & \color{gray}0.8497 {\tiny$\pm 0.0023$} & \color{gray}0.9267 {\tiny$\pm 0.0033$} & \color{gray}0.7846 {\tiny$\pm 0.0018$} \\\hdashline
  \textbf{\color{gray}\cmark} & \textbf{\color{gray}\cmark} & \textbf{\color{gray}\xmark} & \textbf{\color{gray}\xmark} & \textbf{\color{gray}0.7284 {\tiny$\pm 0.0019$}} & \textbf{\color{gray}0.7670 {\tiny$\pm 0.0027$}} & \textbf{\color{gray}0.8826 {\tiny$\pm 0.0017$}} & \textbf{\color{gray}0.6782 {\tiny$\pm 0.0066$}} \\\hdashline
  \textbf{\color{gray}\cmark} & \textbf{\color{gray}\cmark} & \textbf{\color{gray}\cmark} & \textbf{\color{gray}\xmark} & \textbf{\color{gray}0.6524 {\tiny$\pm 0.0009$}} & \textbf{\color{gray}0.7276 {\tiny$\pm 0.0028$}} & \textbf{\color{gray}0.8821 {\tiny$\pm 0.0032$}} & \textbf{\color{gray}0.6192 {\tiny$\pm 0.0025$}} \\
  \bottomrule
  \end{tabular}
  \vspace{0.4cm}
  \caption{Classification results. The used, corresponding parts of the reference designation system are marked as such (\cmark/\xmark). We tested with $10$-fold cross-validation and report the mean values of the measurements with associated standard deviation. The experiments with complete/incomplete identifiers are colored ({$\bullet$}/{\color{gray}$\bullet$}), respectively.}
  \label{results}
  \vspace{-0.4cm}
\end{table}

The reference designation system for power plants is a labeling system consisting of four levels of structure. First, the overall plant is designated by the use of a letter (L) or a digit (D).  The second level of detail designates a higher-level functional system in the overall plant and consists of three letters and two digits with an optional leading digit. The third outline level designates an aggregate in the sub-plant. It consists of two letters and three digits. The letters are assigned to the aggregates in power plants (e.g. a measurement) according to a certain key. The fourth subdivision level designates a device or a signal indicator in the aggregate. It consists of two letters and two digits. The letters are assigned to a piece of equipment (e.g. a drive) according to a given key:
\begin{align*}
  \overbrace{\text{L or D}}^{\text{Overall system (OS)}} \quad \overbrace{\text{(D)LLLDD}}^{\text{Function (F)}} \quad \overbrace{\text{LLDDD(L)}}^{\text{Aggregate (A)}} \quad \overbrace{\text{LLDD}}^{\text{Operating resources (OR)}}.
\end{align*}
In each case, we classify all training samples, each drawn without replacement from one sensor signal recorded over one year. After transforming the sample with the parameters determined from §\ref{timedelay} and §\ref{embeddingdimension} using Takens' embedding, we compute the Betti curves for $\beta_0$ and $\beta_1$ of its persistence diagram. Afterwards, we use them as input to the sub-nets without batch normalization as shown in Fig. \ref{architecture}. Thus, we enrich the neural network input with information about connectedness and the underlying $N$-dimensional torus of the embedded signal.

We summarize the results from Tab. \ref{results} as follows:
\begin{enumerate}
  \item The best classification results for the power plant reference designation system have been measured with an accuracy of approximately $64\%$ (OS F A OR), about $83\%$ for the \emph{overall system} (OS), $73\%$ for the \emph{functional level} (OS F), and $65\%$ for the \emph{aggregate} (OS F A) on the validation data.
  \item We observe for all experiments that precision is higher than recall. We interpret the high precision as a solid exactness of our classifiers. Recall, on the other hand, is significantly lower, as an indicator of the completeness of the classifier. The higher the recall value, the less accurate our classifiers are at assigning signals to the corresponding labels. This can already be anticipated from the Betti curves shown in Fig. \ref{persrepr}, which reveal similar features for the zeroth/first persistent homology group for some of the signals, symptomatic for their poor distinctiveness.
  \item We have shown, that residual networks improve the classification results for all labels except for the assignment to the \emph{overall system} (OS). Further investigation is needed to find an explanation for this behavior.
  \item Moreover, using $\beta_0$ and $\beta_1$-curves we could improve the expected value of the classification results for all label variants studied, see Tab. \ref{results}.
\end{enumerate}

\section{Summary}
\label{discussion}
In this work, we used Betti curves of the zeroth and first persistent homology group as feature vectors for neural networks. Using recent research, we have shown how these two persistent homology groups naturally `\emph{encode events}' in (quasi-)periodic time series and justified their suitability mathematically. We used a variant of Perea's framework for the analysis of (quasi-)periodicity, and applied it as a feature generator for machine learning on irregular time series of measurements taken from power plants. The topological features obtained through Takens' delay
embedding and persistent homology are used to enrich the input of a deep neural network. We have shown that these extra features significantly improve the network's performance in a classification task.

Moreover, we have shown that the experiments also reveal an improvement in classification through $\beta_0$ and $\beta_1$-curves for all hierarchy levels of the power plant reference designation system. We designed the architecture using residual connections and were able to confirm their usefulness compared to the same architecture without such connections.

Our future work will address the following open questions:
\begin{itemize}
  \item How well does our model perform when trained on a larger and more complete data set with signals from multiple entire power plants, but validated with sensor data from an independent new power plant?
  \item Can we improve the classifier by first training it at the coarsest hierarchical level of the power plant reference designation system (OS) and then using the resulting weights to initialize training for more detailed levels of labels (such as F,A) down to the operating resource (OR) in a recursive manner?
\end{itemize}

%
%
%
\bibliographystyle{splncs04}
\bibliography{mybibliography}
\end{document}